\documentclass[letterpaper, 10 pt, conference]{ieeeconf}  

\IEEEoverridecommandlockouts                              

\usepackage{amsmath,amssymb,amsfonts}
\DeclareMathOperator*{\argmax}{arg\,max}
\DeclareMathOperator*{\argmin}{arg\,min}

\newtheorem{theorem}{Theorem}

\newtheorem{assumption}{Assumption}
\newtheorem{definition}{Definition}

\newtheorem{problem*}{Problem}

\usepackage[linesnumbered]{algorithm2e}
\usepackage{graphicx}
\usepackage{textcomp}
\usepackage{xcolor}
\usepackage{breqn}
\usepackage{tikz}
\usetikzlibrary{shapes,arrows}
\usepackage{bm,times}
\usetikzlibrary{arrows, decorations.markings}
\usepackage[noadjust]{cite}

\tikzstyle{vecArrow} = [thick, decoration={markings,mark=at position
   1 with {\arrow[semithick]{open triangle 60}}},
   double distance=1.4pt, shorten >= 5.5pt,
   preaction = {decorate},
   postaction = {draw,line width=1.4pt, white,shorten >= 4.5pt}]
\tikzstyle{innerWhite} = [semithick, white,line width=1.4pt, shorten >= 4.5pt]

\title{\LARGE \bf Probabilistic Data Association for Semantic SLAM at Scale}

\author{Elad Michael$^{1}$, Tyler Summers$^{2}$, Tony A. Wood$^{3}$, Chris Manzie$^{1}$, and Iman Shames$^{4}$
\thanks{*This work was not supported by any organization}
\thanks{$^{1}$Department of Electrical and Electronic Engineering, University of Melbourne
        {\tt\small \{eladm@student,manziec@,\}unimelb.edu.au}}%
\thanks{$^{2}$Control, Optimization, and Networks Lab, University of Texas at Dallas
        {\tt\small tyler.summers@utdallas.edu}}%
\thanks{$^{3}$SYCAMORE Lab, Ecole Polytechnique Federale de Lausanne (EPFL), Lausanne, Switzerland
        {\tt\small tony.wood@epfl.ch}}%
\thanks{$^{4}$CIICADA Lab, School of Engineering, Australian National University, Acton, ACT, 0200, Australia
        {\tt\small iman.shames@anu.edu.au}}%
}

\begin{document}

\maketitle
\thispagestyle{empty}
\pagestyle{empty}

\begin{abstract}
With advances in image processing and machine learning, it is now feasible to incorporate semantic information into the problem of simultaneous localisation and mapping (SLAM). Previously, SLAM was carried out using lower level geometric features (points, lines, and planes) which are often view-point dependent and error prone in visually repetitive environments. Semantic information can improve the ability to recognise previously visited locations, as well as maintain sparser maps for long term SLAM applications. However, SLAM in repetitive environments has the critical problem of assigning measurements to the landmarks which generated them. In this paper, we use k-best assignment enumeration to compute marginal assignment probabilities for each measurement landmark pair, in real time. We present numerical studies on the KITTI dataset to demonstrate the effectiveness and speed of the proposed framework.
\end{abstract}


\section{Introduction}\label{sec:intro}

In 1956, a fried chicken restaurant in Marietta, Georgia constructed a ``... 56-foot-tall (17 m) steel-sided structure designed in the appearance of a chicken rising up from the top of the building.'' The so-called ``Big Chicken'' has become one of the most well recognised landmarks of the area. When navigating the town with locals, one may be directed to ``turn left after the Big Chicken'' or ``head one mile south of the Big Chicken''. By virtue of being a uniquely identifiable semantic landmark, the Big Chicken simplifies navigation and localisation in Marietta, Georgia. 

Recently, with advances in machine learning, semantic object recognition has become fast and reliable enough to incorporate into autonomous navigation algorithms. A semantic measurement, usually extracted from a photograph or a lidar point cloud, is a combination of a state estimate to localise the object and a semantic \emph{category}, e.g. car, person, or Big Chicken. Semantic maps have been of interest for more than a decade\cite{nuchter2008towards}, but earlier works are focused on the production of a semantic map for the user and do not incorporating the semantic information into the state estimation problem. While semantic landmarks do create more human-friendly maps, they also impart rotational and translational accuracy~\cite{doherty2020probabilistic}, filter out moving landmarks\cite{wang2019computationally}, and provide a natural loop closure mechanism due to their relative sparsity in the environment. However, even with their relative sparsity compared to geometric features such as ORB~\cite{rublee2011orb} descriptors, any successful SLAM algorithm must be robust to incorrect data associations.

There is a litany of available sensors and measurement types, including but not limited to: inertial measurements, photographic features, radar, and lidar. Regardless of measurement type however, all landmark based simultaneous landmark and mapping (SLAM) methods must be able to attribute new measurements to existing landmarks, i.e. solve the data association/assignment problem. Early SLAM algorithms assumed measurements were \emph{identifiable}~\cite{leonard1991mobile}, i.e. could be unambiguously attributed to the landmark which generated them. However, in a real world application, a SLAM implementation must be able to correctly assign measurements to landmarks as well as be robust to incorrect and ambiguous assignments. Measurements which are incorrectly assigned are not just wasted information, incorrect assignments introduce an incorrect bias and are the most common cause of algorithmic failure in SLAM~\cite{cadena2016past}. 

Many semantic SLAM implementations focus on the development of a novel feature, i.e. quadric or mesh state information, and simply use some method of maximum likelihood assignment~\cite{nicholson2018quadricslam,yang2019cubeslam,hosseinzadeh2018structure}. However, rather than rely on the feature extraction to approach the perfectly identifiable ideal, we focus on SLAM algorithms which are designed to be robust to noisy assignment data. A popular approach to assignment robustness is \emph{multiple hypothesis tracking} (MHT), in which a set of viable assignment hypotheses is maintained rather than the single maximum likelihood assignment. As new measurements are introduced, the set of hypotheses is branched to include additional assignments. In order to remain computationally feasible, multiple hypothesis tracking methods, as well as particle filter approaches~\cite{montemerlo2003fastslam}, require constant pruning of the set of hypotheses with statistical tests or weighted random sampling of the viable hypothesis pool. An example of semantic SLAM using MHT can be found in~\cite{bernreiter2019multiple}. The authors of~\cite{lu2021consensus} maintain the distribution of assignments at each keyframe, and use the modal probability assignment under the linear assignment model at each iteration for estimation. This allows for assignment reconsideration without parallel estimators like in MHT, but may be slow to modify the factor graph with the modal assignment at each iteration. 

Most relevant to this work, recently the authors of~\cite{doherty2019multimodal,bowman2017probabilistic} have used probabilistic assignment models to represent ambiguity with a single estimation rather than maintaining parallel estimations for each assignment. The weight given to each assignment in the estimation is computed by marginalising over the distribution of all assignments. However, the number of assignments grows combinatorially, and explicitly marginalising over the set of possible assignments is only feasible in small cases. In~\cite{bowman2017probabilistic}, the authors show that computing the marginalisation over the assignments is equivalent to computing the matrix permanent, using the linear assignment model. However, computing the matrix permanent is in itself a a $\#P-$complete problem. Approximation methods are included, but do not scale well with the size of the assignment problem. 

\paragraph*{Contributions} In this paper, we propose a method to approximate the marginalisation over the set of possible assignments using ranked assignments. This technique was popularized in multi-target tracking (Section 6.6 of~\cite{blackman1999design}), but has not yet been applied to probabilistic data association in semantic SLAM. In this work,
\begin{itemize}
    \item We construct a bound on the difference between the approximated marginalisation probabilities and the truth.
    \item We show that the ranked assignment approximation with $200$ assignments is as fast as the matrix permanent on small problems (fewer than $\approx 15$ rows/cols), and is orders of magnitude faster on larger problems.
    \item We show that using the top $200$ assignments was generally \emph{more} accurate than the fastest permanent approximation method over the extracted KITTI data set problems.
    \item We provide the code used to generate the semantic assignment datasets as well as for computing the assignments, matrix permanent, and generating the result graphs. 
\end{itemize}

Finally, the framework we propose is agnostic in the method of generating the likeliest assignment from a set of measurements and landmarks, making this approach flexible to using nearest neighbour, linear sum assignment~\cite{jonker1987shortest}, or other methods such as joint compatibility branch and bound~\cite{tardos2001data} to determine the likeliest assignment.

\section{Probabilistic SLAM Background}\label{sec:probSLAM}

We differentiate between three categories of measurements: inertial, geometric, and semantic. Inertial measurements do not require data association, as they are generated by an onboard inertial measurement unit and use sensors such as accelerometers, gyroscopes, and wheel trackers to estimate the robots motion. Geometric measurements such as ORB\cite{rublee2011orb} descriptors, lidar, or radar are the primary tool of many SLAM implementations, and do require data association. Geometric measurements are generated by landmarks in the environment, but are high density and relatively non-unique. Data association with geometric measurements focuses on speed and leverages nearest neighbour type associations with high rejection thresholds~\cite{orbslam3}. Using probabilistic data association methods on these high density and high frequency measurements is not currently computationally feasible or necessary. Semantic measurements are significantly more sparse than geometric measurements, as well as less dependent on viewpoint or lighting conditions. Therefore, we only apply the probabilistic data assignment to semantic measurements, to ensure robust and correct data association and loop closures.

We begin with a probabilistic characterisation of the SLAM problem, 
\begin{align}
\mathcal{X}^*,\mathcal{L}^* = \argmax_{\mathcal{X},\mathcal{L},\mathcal{A}} \log p(\mathcal{Z} \mid \mathcal{X},\mathcal{L},\mathcal{A}),\label{eq:SLAMwAss}
\end{align}
where $\mathcal{X}$ represents the sequence of robot states, $\mathcal{L}$ the set of landmark states, $\mathcal{A}$ encodes the assignment of measurements to landmarks, and $\mathcal{Z}$ the sequence of measurements. We will use $z_k\in\mathcal{Z}$ to mean the $k$-th \emph{feature} measurement, i.e. a semantic/geometric feature extracted from a photograph or an inertial measurement extracted from an IMU.
\begin{assumption}\label{ass:oneLandPerMeas}
Each measurement $z_k\in\mathcal{Z}$ is generated by \emph{at most one} landmark. 
\end{assumption}
A common approach to solve~\eqref{eq:SLAMwAss} is to separate the global estimation into a continuous pose estimation~\eqref{eq:poseGivenAss} and a discrete assignment estimation problem~\eqref{eq:assGivenPose}. An initial estimate of the robot states and landmarks $\mathcal{X}^{(0)}$ and $\mathcal{L}^{(0)}$ is used to compute the assignment $\mathcal{A}^{(0)}$, which is in turn taken to compute a new estimate of the robot states and landmarks $\mathcal{X}^{(1)}$ and $\mathcal{L}^{(1)}$. This type of coordinate descent, see~\cite{fatemi2015variational} for example, iterates between
\begin{align}
\mathcal{X}^{(i+1)},\mathcal{L}^{(i+1)} &= \argmax_{\mathcal{X},\mathcal{L}} \log p(\mathcal{Z} \mid \mathcal{X},\mathcal{L},\mathcal{A}^{i}), \label{eq:poseGivenAss}\\
\mathcal{A}^{(i+1)} &= \argmax_{\mathcal{A}} p(\mathcal{A} \mid \mathcal{X}^{(i+1)},\mathcal{L}^{(i+1)},\mathcal{Z}), \label{eq:assGivenPose}
\end{align}
until the estimates converge or for some fixed number of iterations. However, this method does not consider the possibility of ambiguous associations, which may corrupt filter performance in the future, especially with filtering approaches\cite{castellanos2007robocentric}.

In~\cite{atanasov2016localization}, the authors argue that a better approach is to maximise the expectation over all possible associations rather than committing to any particular assignment. We present a summary of their analysis here, beginning with the expectation over the associations
\begin{align}
&\mathcal{X}^{i+1},\mathcal{L}^{i+1} = \argmax_{\mathcal{X},\mathcal{L}} \mathbb{E}_{\mathcal{A}} [\log p(\mathcal{Z} \mid \mathcal{X},\mathcal{L},\mathcal{A})\mid \mathcal{X}^{i}, \mathcal{L}^{i}, \mathcal{Z}]. \nonumber \\
&= \argmax_{\mathcal{X},\mathcal{L}} \sum_{\mathcal{A}\in\mathbb{A}} p(\mathcal{A} \mid \mathcal{X}^i, \mathcal{L}^i, \mathcal{Z}) \log p(\mathcal{Z} \mid \mathcal{X},\mathcal{L},\mathcal{A}),\nonumber\\
\begin{split}
&= \argmax_{\mathcal{X},\mathcal{L}}  \sum_{k=1}^{|\mathcal{Z}|} \sum_{j=1}^{|\mathcal{L}|} \sum_{\mathcal{A}\in\mathbb{A}_{kj}} [\; p(\mathcal{A} \mid \mathcal{X}^i, \mathcal{L}^i, \mathcal{Z}) \\
&\qquad \qquad \qquad \qquad \qquad \qquad \log p(z_k \mid x_{\alpha_k},l_j) \; ],\label{eq:expctExpansion}
\end{split}
\end{align}
where $\mathbb{A}_{kj}$ denotes the subset of all associations which assign measurement $z_k$ to landmark $l_j$, and $\alpha_k$ the index of the robot state from which $z_k$ was measured. Note that this is still an iterative process, as the probability of an association $A\in\mathcal{A}$ is a function of the given robot and landmark states. Finally, we may factor the log term outside of the innermost summation, as the assignment of $k$ to $j$ is constant within the subset of assignments in order to define
\begin{align}
w^i_{kj} &:= \sum_{\mathcal{A}\in\mathbb{A}_{kj}}  p(\mathcal{A} \mid \mathcal{X}^i, \mathcal{L}^i, \mathcal{Z}). \label{eq:probOfAssign}\\
\intertext{Substituting back into~\eqref{eq:expctExpansion},}
\mathcal{X}^{i+1},\mathcal{L}^{i+1} &= \argmax_{\mathcal{X},\mathcal{L}}  \sum_{k=1}^{|\mathcal{Z}|} \sum_{j=1}^{|\mathcal{L}|} w^i_{kj} \log p(z_k \mid x_{\alpha_k},l_j). \label{eq:expectationSLAM}
\end{align}

The expectation allows no hard decisions to be made over the assignments, and considers the full distribution of possible assignments. However, expectation maximisation introduces a new challenge, namely the computation of~\eqref{eq:probOfAssign}. The number of assignments grows \emph{combinatorially}, and in practice cannot be enumerated for anything other than single digit quantities of landmarks and measurements in real time. In~\cite{bowman2017probabilistic} they prove that, given some simplifying assumptions, the computation of~\eqref{eq:probOfAssign} is equivalent to computing the matrix permanent. However, computing the matrix permanent is itself a $\#P-$complete problem, with computational complexity $\mathcal{O}(d2^{d-1})$, or $\mathcal{O}(d^7(\log(d))^4)$ to compute an approximation of the permanent for $d\times d$ matrix. 

\section{K-Best Assignment Enumeration}\label{sec:primaryResults}

In this section we present our primary result, a fast method to approximate the probability of assignment $w^{i}_{kj}$ from~\eqref{eq:probOfAssign}, as well as a bound on the error of the approximation. We begin with a brief discussion of ranked assignments, for a detailed discussion we refer to~\cite{murty1968letter}. For a given $\mathcal{Z},\mathcal{X}^{i},\mathcal{L}^{i}$, define 
\begin{align}
p(\mathcal{A}) := p(\mathcal{Z} \mid \mathcal{X}^{i},\mathcal{L}^{i},\mathcal{A}). \label{eq:probShorthand}
\end{align}
We may then define an ordering over a set of possible assignments $\mathbb{A}$, such that 
\begin{align}
p(\mathcal{A}_1) \geq p(\mathcal{A}_2) \geq p(\mathcal{A}_3) \geq ... \;\forall\; \mathcal{A}_i\in\mathbb{A}. \label{eq:assignmentRanking}
\end{align}
The methods to construct this ordering may change depending on the model used to determine the likelihood $p(\mathcal{Z} \mid \mathcal{X}^{i},\mathcal{L}^{i},\mathcal{A})$. For example, Murty's algorithm~\cite{murty1968letter,miller1997optimizing} works with any graph based assignment method but is typically associated with the linear sum assignment model. In Section~\ref{sec:implementation}, we use Murty's algorithm combined with the linear assignment method but leave the discussion in this section general to any method of generating ranked assignments.

To approximate the marginal assignment probability, we begin with the $w_{kj}^i$ for a given $\mathcal{Z},\mathcal{X}^{i},\mathcal{L}^{i}$
\begin{align}
w^{i}_{kj} &:= \sum_{\mathcal{A}\in\mathbb{A}_{(k,j)}} p(\mathcal{A} \mid \mathcal{X}^{i},\mathcal{L}^{i},\mathcal{Z}),\nonumber\\
&= \frac{\sum_{\mathcal{A}\in\mathbb{A}_{(k,j)}} p(\mathcal{Z} \mid \mathcal{X}^{i},\mathcal{L}^{i},\mathcal{A})p(\mathcal{A}\mid\mathcal{X}^{i},\mathcal{L}^{i})}{\sum_{\mathcal{A}\in\mathbb{A}} p(\mathcal{Z} \mid \mathcal{X}^{i},\mathcal{L}^{i},\mathcal{A})p(\mathcal{A}\mid\mathcal{X}^{i},\mathcal{L}^{i})}.\nonumber\\
\intertext{Assuming $p(\mathcal{A}\mid\mathcal{X}^{i},\mathcal{L}^{i})$ is uniform, as assumed in~\cite{bowman2017probabilistic}, }
w^{i}_{kj} &= \frac{\sum_{\mathcal{A}\in\mathbb{A}_{(k,j)}} p(\mathcal{A})}{\sum_{\mathcal{A}\in\mathbb{A}} p(\mathcal{A})},\label{eq:truewij}
\end{align} 
where we have $\mathbb{A}$ representing the set of all possible assignments, $\mathbb{A}_{(k,j)}$ the subset of assignments which assign measurement $k$ to landmark $j$, and $p(\mathcal{A})$ as defined in~\eqref{eq:probShorthand}.

Let $\overline{\mathbb{A}} := \{\mathcal{A}_i\}_{i=1}^{K}$ be the set of $K$ likeliest assignments. We define the approximate association probability $\overline{w}^{i}_{kj}$ as 
\begin{align}
\overline{w}^{i}_{kj} := \frac{\sum_{\mathcal{A}\in\overline{\mathbb{A}}_{(k,j)}} p(\mathcal{A})}{\sum_{\mathcal{A}\in\overline{\mathbb{A}}} p(\mathcal{A})},\label{eq:approxwij}
\end{align}
where we have mirrored the previous notation replacing the set of $K$ likeliest hypotheses $\overline{\mathbb{A}}$ for $\mathbb{A}$. It is important to note that $\overline{\mathbb{A}}_{(k,j)}\subseteq\overline{\mathbb{A}}$ is the subset of $\overline{\mathbb{A}}$ which assigns measurement $k$ to landmark $j$, not the $K$ likeliest assignments which assign measurement $k$ to landmark $j$. 

\subsection{Marginal Assignment Probability Error Bound}

In order to quantify the approximation error, i.e. establish a relationship between~\eqref{eq:truewij} and~\eqref{eq:approxwij}, we define the following quantity based on the set of $k-$likeliest assignments. 

\begin{definition}[Association Error Bound]\label{def:assErrBound}
For a given set of measurements $\mathcal{Z}$ with a state and landmark prior $\mathcal{X}^{i},\mathcal{L}^{i}$, let $\overline{\mathbb{A}}$ be the set of $K$ likeliest associations. We define the \emph{association error bound} 
\begin{align}
\gamma(\overline{\mathbb{A}}) &:= \frac{\beta(\overline{\mathbb{A}})}{\beta(\overline{\mathbb{A}}) + \sum_{\mathcal{A}\in\overline{\mathbb{A}}} p(\mathcal{A})}\;, \label{eq:threshold}
\intertext{for some $\beta:2^{\mathcal{A}}\rightarrow\mathbb{R}^+$ satisfying}
\beta(\overline{\mathbb{A}}) &\geq \sum_{\mathcal{A}\in{\mathbb{A}\setminus\overline{\mathbb{A}}}} p(\mathcal{A}), \label{eq:remProbBound}
\end{align}
with $p(\mathcal{A})$ as defined in~\eqref{eq:probShorthand}, and $2^{\mathcal{A}}$ the power set of $\mathcal{A}$.
\end{definition}

Intuitively, we may describe $\beta(\overline{\mathbb{A}})$ as a bound on the total probability not accounted for in $\overline{\mathbb{A}}$, where $\gamma(\overline{\mathbb{A}})$ represents the same quantity as a fraction of the estimated total. With the association bound $\gamma(\overline{\mathbb{A}})$, we prove the following.

\begin{theorem}\label{thm:probNeighbourhood}
For a given set of measurements $\mathcal{Z}$ with a state and landmark prior $\mathcal{X}^{i},\mathcal{L}^{i}$, let $\overline{\mathbb{A}}$ be the set of $k$ highest probability associations and $\gamma(\overline{\mathbb{A}})$ as in Definition~\ref{def:assErrBound}. Then the estimated and true association probabilities $\bar{w}_{kj}^{i},w_{kj}^{i}$, defined in~\eqref{eq:approxwij} and~\eqref{eq:truewij}, satisfy
\begin{align}
w_{kj}^{i} \in [\bar{w}_{kj}^{i}-\gamma(\overline{\mathbb{A}}),\bar{w}_{kj}^{i}+\gamma(\overline{\mathbb{A}})].
\end{align}
\end{theorem}

\begin{proof}
See Appendix~\ref{app:probNeighbourhood}.
\end{proof}

To compute the bound $\beta(\overline{\mathbb{A}})$ from~\eqref{eq:remProbBound}, note that all assignments in $\mathbb{A}\setminus\overline{\mathbb{A}}$ have probability upper bounded by the $k$-th likeliest hypothesis enumerated i.e.
\begin{align}
p(\mathcal{A}_{K}) \geq p(\mathcal{A}), \quad \forall \; \mathcal{A}\in \mathbb{A}\setminus\overline{\mathbb{A}}.
\end{align}

Therefore, for given a bound on the number of assignments $N \geq |\mathbb{A}\setminus\overline{\mathbb{A}}|$, we may set $\beta(\overline{\mathbb{A}}) = N\cdot p(\mathcal{A}_{K})$. To bound the number of remaining assignments, we use the $\{0,1\}$ matrix permanent upper bound from~\cite{minc1963upper}. This permanent approximation has linear computational complexity in the number of rows, and is only used to compute the association error bound $\gamma(\overline{\mathbb{A}})$.

\section{Implementation}\label{sec:implementation}
\pgfdeclarelayer{background}
\pgfdeclarelayer{foreground}
\pgfsetlayers{background,main,foreground}

\tikzstyle{textBox}=[draw, text width=10em, 
    text centered, minimum height=4em, minimum width=5em,rounded corners]
\def\blockdist{2.3}
\def\edgedist{2.5}
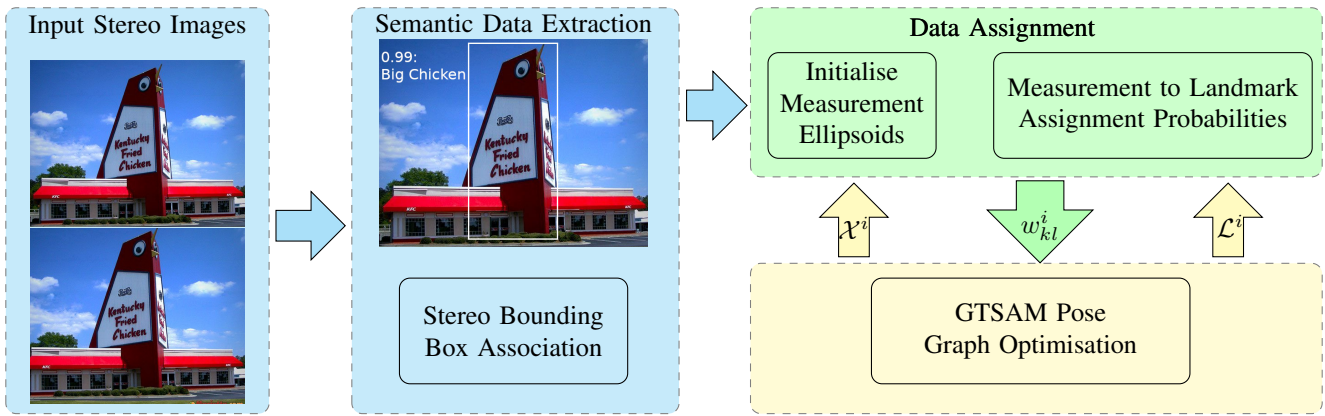
\begin{figure*}[thpb]

\begin{tikzpicture}
    \node (imgL) [yshift = -0.9cm, xshift=-1cm]{\includegraphics[width=0.16\textwidth]{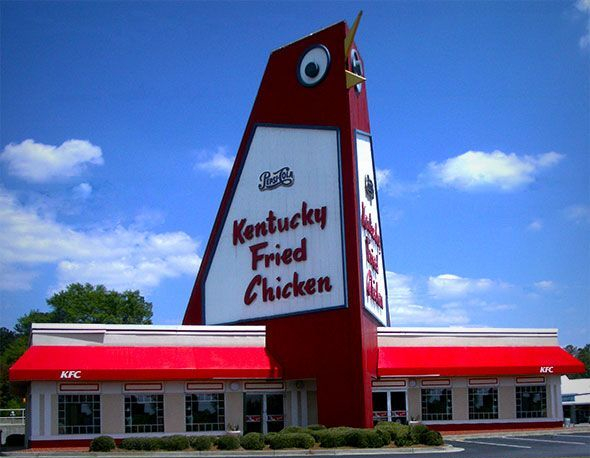}};
    \node (imgR) [below of=imgL, yshift = -1.3cm]{\includegraphics[width=0.16\textwidth]{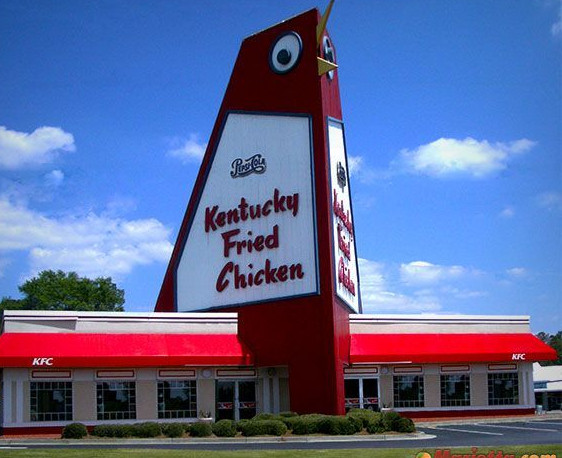}};

    \node (txt1) [above of=imgL, yshift=0.55cm] {Input Stereo Images};
    
    \node (pass) [single arrow] [right of=imgL, draw,  fill=cyan!30, , yshift=-1.1cm, xshift=1.2cm] {\phantom{zzz}};

    \node (semLabel) [right of=imgL, xshift = 4cm, yshift = 0cm] {\includegraphics[width=0.2\textwidth]{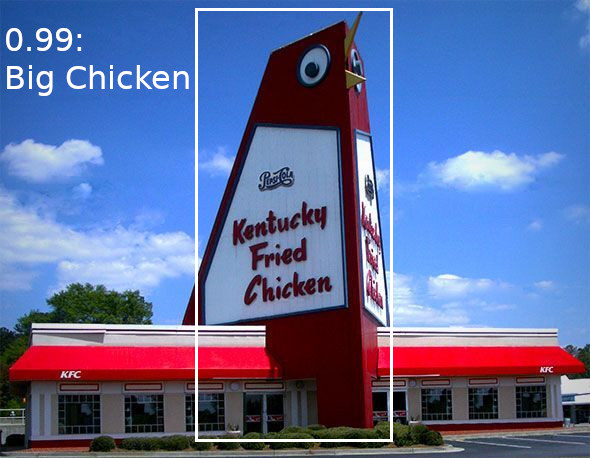}};
    \node (extract) [textBox] [fill=cyan!20, below of=semLabel, xshift = 0cm, yshift = -1.5cm,text width = 2.8cm] {Stereo Bounding Box Association};
    \node (txt2) [above of=semLabel, yshift=0.6cm] {Semantic Data Extraction};

    \node (pass) [single arrow] [right of=semLabel, draw, fill=cyan!30, yshift=0.5cm, xshift=1.6cm] {\phantom{xx}};

    \node (initQ) [textBox] [right of=semLabel, xshift = 3.5cm, yshift = 0.5cm,text width = 2cm] {Initialise Measurement Ellipsoids};

    \node (assoc) [textBox] [right of=initQ, xshift = 3cm, yshift = 0cm,text width = 4cm] {Measurement to Landmark Assignment Probabilities};

    \node (txt3) [above of=assoc, xshift=-2cm, yshift=0cm] {Data Assignment};

    \node (poseGraph) [textBox] [right of=extract, xshift = 5.875cm, yshift = 0cm,text width = 4cm] {GTSAM Pose Graph Optimisation};

    \node (txt3) [above of=assoc, xshift=-2cm, yshift=0cm] {Data Assignment};

    \node (pass3) [single arrow] [minimum width = 4em, draw, below of=assoc, fill=green!30, yshift=-0.39cm, xshift=-1.5cm, rotate=-90] {\phantom{xxx}};
    \node (pass3Text) [below of=assoc, yshift=-0.6cm, xshift=-1.5cm] {$w_{kl}^{i}$};

    \node (pass2) [single arrow] [above of=poseGraph, draw, fill=yellow!30, yshift=0.375cm, xshift=-2.35cm, rotate=90] {\phantom{xxx}};
    \node (pass2Text) [above of=poseGraph, yshift=0.375cm, xshift=-2.35cm] {$\mathcal{X}^{i}$};

    \node (pass3) [single arrow] [below of=assoc, draw, fill=yellow!30, yshift=-0.625cm, xshift=1cm, rotate=90] {\phantom{xxx}};
    \node (pass3Text) [below of=assoc, yshift=-0.6cm, xshift=1.03cm] {$\mathcal{L}^{i}$};

    \begin{pgfonlayer}{background}
        \path[fill=cyan!20,rounded corners, draw=black!50, dashed]
            (-2.75,-4.5) rectangle (0.75,0.9);
        \path[fill=cyan!20,rounded corners, draw=black!50, dashed]
            (1.85,-4.5) rectangle (6.2,0.9);

        \path[fill=green!20,rounded corners, draw=black!50, dashed]
            (7.15,-1.35) rectangle (14.75,0.9);

        \path[fill=yellow!20,rounded corners, draw=black!50, dashed]
            (7.15,-4.5) rectangle (14.75,-2.5);



    \end{pgfonlayer}
\end{tikzpicture}
\caption{Pipeline for semantic SLAM data extraction. \label{fig:slamPipeline}}
\end{figure*}

A C++ implementation of all the code and data used to generate the following results are freely available on github\footnote{ https://github.com/EladMichael/probabilisticSemSlam }. For the SLAM data, we use the KITTI odometry data set\cite{kitti}. For semantic bounding box extraction, we use Yolov5\cite{glenn_jocher_2022_6222936}, which provides a small enough network to extract measurements from 2 images in $~50ms$ for both on a CPU. At each time step we extract semantic bounding boxes, as shown in Figure~\ref{fig:slamPipeline}, and associate the bounding boxes between the image pair. To represent the semantic landmarks, we use the dual quadric framework from~\cite{nicholson2018quadricslam}. Each landmark is modelled as an ellipsoid in three dimensions, which can be compactly represented as a $4\times4$ positive semi-definite matrix $Q$. For complete details, refer to~\cite{nicholson2018quadricslam}. Given two such ellipsoids $Q_1,Q_2$, we define the distance function
\begin{align}
d(Q_1,Q_2) := (\mu_1 - \mu_2)(P_1 + P_2)^{-1}(\mu_1-\mu_2),\label{eq:bhatDist}
\end{align}
where $\mu_i,P_i$ are the mean and covariance matrix of the ellipsoid $i$, which can be constructed from the matrix $Q_i$. To turn the measured semantic bounding boxes into ellipsoids for assignment, we use simple stereo triangulation to generate a $3D$ ellipsoid representing the detected object. Note that the ellipsoid generated in this stage is only used to calculate the assignment probability, the landmark parameters are constrained directly by the bounding boxes in the resulting pose graph. Using the measurement and landmark ellipsoids, we can construct all of the pairwise distances~\eqref{eq:bhatDist}, and compute the association probabilities $\bar{w}^{i}_{kj}$ for all measurements $z_k$ and landmarks $l_j$. We also include a ``null'' association, with a fixed cost, which is interpreted as a non-assignment and results in the measurement constructing a new landmark. After computing the association probabilities~\eqref{eq:approxwij}, we use the bounding boxes extracted from the data set, as well as the provided ground truth odometry, to construct the non-linear least squares problem
\begin{align}
\begin{split}
\mathcal{X}^{i},\mathcal{L}^{i} &= \argmin_{\mathcal{X},\mathcal{L}} \underbrace{\sum_{t}||f(x_{t},u_{t})\ominus x_{t+1}||^2_{R_o}}_{\textrm{odometry factors}} \\
&+ \underbrace{\sum_{k,l}||b_{k} - \beta_{(x_{\alpha_k},q_l)}||^2_{R_s/\bar{w}^{i}_{kj}}}_{\textrm{semantic factors}}, \label{eq:poseGraph}
\end{split}
\end{align}
where $R_o,R_s$ are the odometry and semantic covariance matrices, $f(x_{t},u_{t})\ominus x_{t+1}$ the predicted robot pose error computed in SE$(3)$, $||b_{k} - \beta_{(x_{\alpha_k},q_l)}||$ the geometric error described in~\cite{nicholson2018quadricslam}, and $\bar{w}^{i}_{kj}$ the approximate association probabilities from~\eqref{eq:approxwij}.

For our SLAM back end to solve~\eqref{eq:poseGraph}, we use GTSAM~\cite{dellaert2012factor}, which provides real time performance for large incremental pose graph optimisations. The estimates are computed using a sliding window of $50$ poses, to model a real-time SLAM algorithm. Our focus in this work was the extraction of realistic semantic assignment problems, so the odometry provided to the algorithm was the ground truth, coupled with a tight covariance ($1$cm in translation, $.001$rad in rotation). By using these tight covariances, the pose graph optimisation~\eqref{eq:poseGraph} constructs the semantic landmarks from the assigned measurements without constructing many erroneous landmarks from incorrect pose estimation. 

The experiments were run on a 2016 HP Elitebook, with an $i7$ processor and $8$GB of RAM. We only compare the times required to compute the association probabilities, given the set of pairwise distances~\eqref{eq:bhatDist} from each measurement to landmark. To compare, we ran simulations using several values of $k$ when enumerating the $k$-best assignments, as well as running both the exact and approximate permanent methods used in~\cite{bowman2017probabilistic}. Further, in the implementation of the matrix permanent code, there is a function which attempts to choose the faster of the exact and approximate methods. Rather than relying on this logic, we compute both methods in our test, and choose the fastest for a best case comparison which is labelled ``Fastest Permanent'' in the results. For $k = 200$ enumerated assignments, the timing results from the $~4500$ assignment problems in KITTI sequence $00$ are show in Figure~\ref{fig:speedScatter}. The data points correspond to the log base $10$ of the milliseconds used, i.e. $-3,0,3$ are microseconds, milliseconds, and seconds respectively. In colour, we have indicated the largest dimension of the $n\times m$ assignment problem. We see from Figure~\ref{fig:speedScatter} that the fastest permanent and the k-best assignment enumeration take approximately the same speed for smaller problems, with max dimension$\approx 15$. After this point however, the fastest permanent method slows down dramatically, reaching $100$s of milliseconds per problem, while the assignment method never breaks $1$ millisecond per problem within this dataset. 

\begin{figure}[thpb]
 \centering
 \includegraphics[scale=0.4]{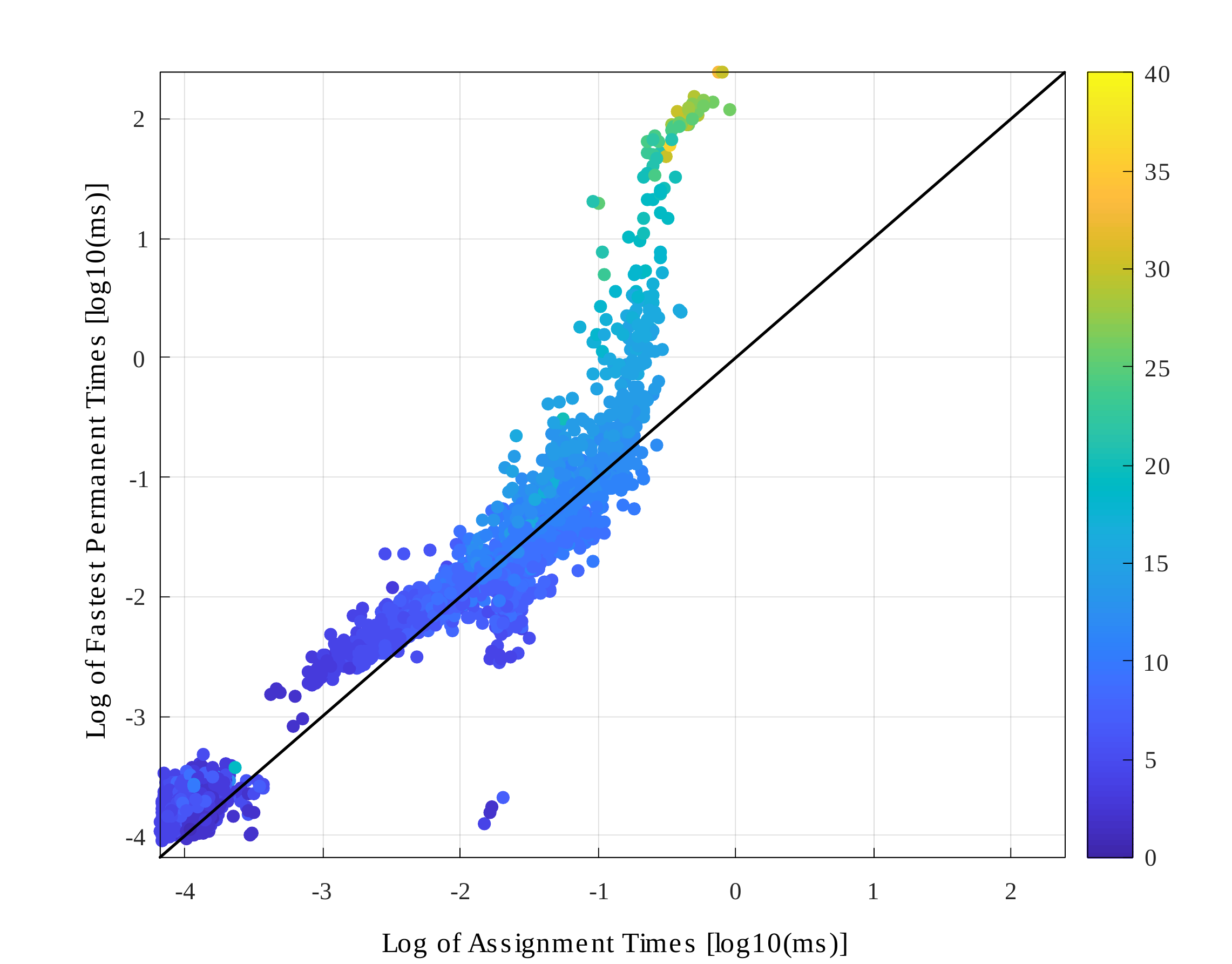}
 \caption{Computation time using the $200$ best assignments and the fastest permanent method. \label{fig:speedScatter}}
\end{figure}

In addition to the speed, we are interested in the accuracy of these approximation methods. To compare them, we computed the complete set of assignment probabilities $p(\mathcal{A}_i)$, in order to compute the probabilities $\{w^i_{kj}\}$ for all measurements $z_k$ and landmarks $l_j$, as in~\eqref{eq:truewij}. If there are more than $20000$ of these terms, we only compute the $20000$ largest. Nominally, this is the definition of the exact permanent, but the algorithm provided from~\cite{bowman2017probabilistic} has numerical overflow issues with large matrices. The brute force method avoids these numerical overflow issues the same way the k-best enumeration does, by operating on the log likelihoods rather than the probabilities. For the error, for an assignment problem at iteration $t$ with given pose and landmark priors $\mathcal{X}^{i},\mathcal{L}^{i}$ we compute
\begin{align}
\delta_t := \max_{k,j}\;|w^{i}_{kj} - \bar{w}^{i}_{kj}|,
\end{align}
where $w^{i}_{kj}$ is the truth association probability, as defined in~\eqref{eq:truewij} and computed by brute force enumeration, and $\bar{w}^{i}_{jk}$ is the approximate probability computed either by permanent or $k$-best assignment enumeration. In Figure~\ref{fig:ordStat}, we show the order statistics of these $\delta_t$ errors from the $\approx 4500$ assignment problems generated from KITTI sequence $0$. As well as computing the top $200$ assignments, as in Figure~\ref{fig:speedScatter}, we include the results from only computing the top $20$ assignments. 
\begin{figure}[thpb]
 \centering
 \includegraphics[scale=0.38]{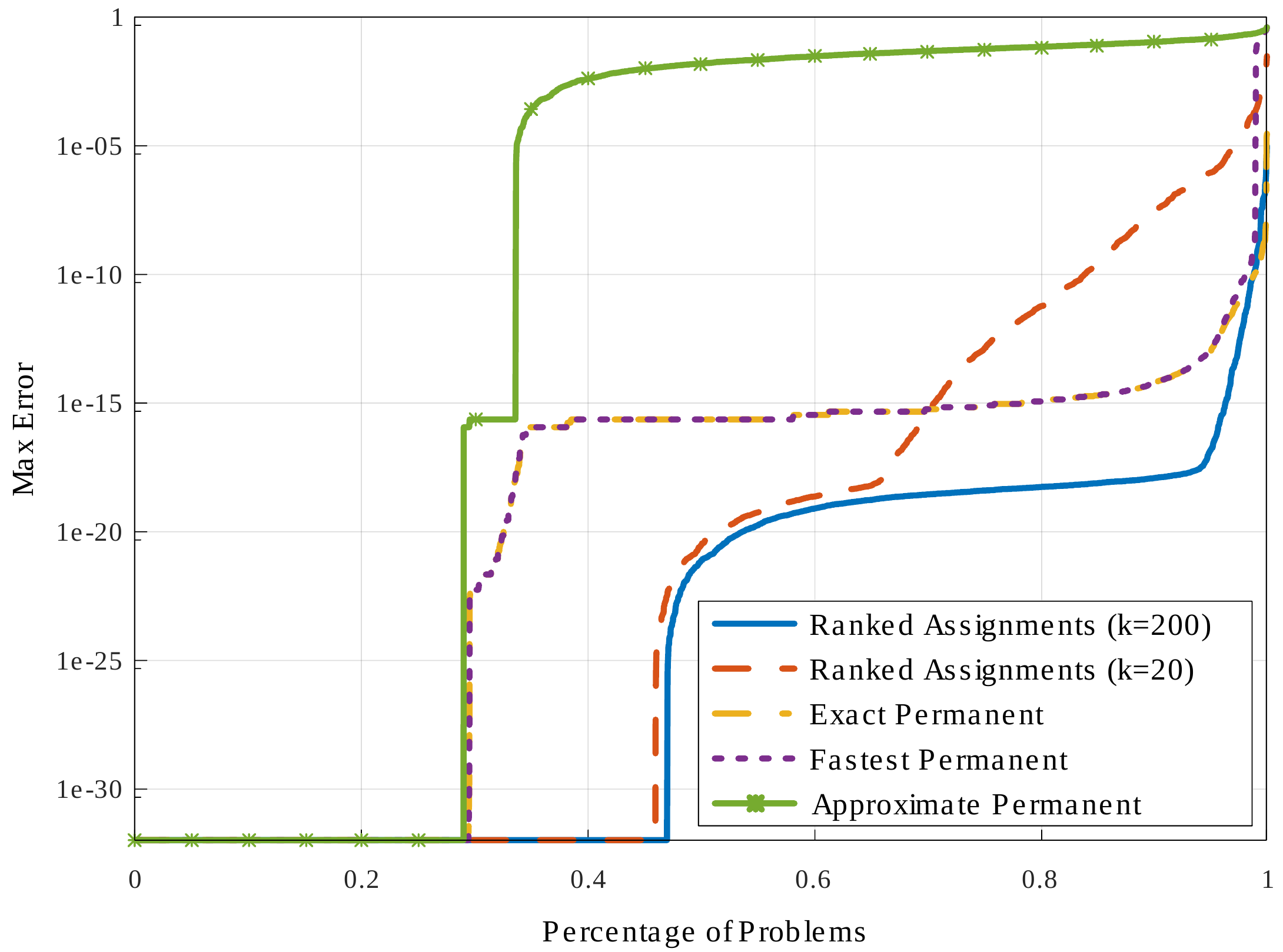}
 \caption{Order statistics of the worst case probability error. \label{fig:ordStat}}
\end{figure}
Note the x-axis of Figure~\ref{fig:ordStat} is normalized so it may be interpreted as a percentage, for example at $0.8$ the ranked assignment with $k=200$ has order statistic of approximately $10^{-17}$ and thus $80$ percent of problems had error equal to or less than $10^{-17}$. For approximately $65\%$ of the problems, there was little difference in using $20$ vs. $200$ assignments, although after that point the $20$ best assignments is excluding somewhat probable assignments, leading to rising errors. Both the $k$-best assignment with $k=200$ and the exact permanent have worst case error of approximately $10^{-5}$, however the fastest permanent method has a significantly higher worst case error, as the exact permanent becomes extremely slow for larger matrices and the approximate permanent becomes the fastest. Furthermore, for nearly $90$ percent of the problems in this data set, the $k$-best assignment with $k=200$ has a worst case error less than $10^{-16}$, which is generally the numerical precision of a floating point number in C++.

\section{Conclusion}\label{sec:conclusion}
In this work, we discussed the probabilistic assignment approach, which allows for robustness to incorrect assignments without maintaining parallel filters. We further show that enumerating the $k$-best assignments is a fast, scalable, and accurate method of computing the probability a measurement was generated by a landmark. This approach is common within the target tracking literature, as discussed in Section~\ref{sec:intro}, and the authors believe it is perfectly suited to the problem of semantic SLAM. In future work, we hope to incorporate probabilistic data association into a complete semantic SLAM framework, incorporate alternative assignment models beyond the linear assignment, and focus on sparse and long term semantic SLAM mapping.

\section*{APPENDIX}

\section{Proof of Theorem~\ref{thm:probNeighbourhood}} \label{app:probNeighbourhood}

This proof is a straight forward, if slightly involved, manipulation of the definitions of $w_{kj}^{i} $ and $ \bar{w}_{kj}^{i}$, given the definitions in~\eqref{eq:truewij}-\eqref{eq:approxwij}, and the definition of $\gamma(\overline{\mathbb{A}})$ given in Definition~\ref{def:assErrBound}. Beginning with the difference of the association probabilities,
\begin{align*}
&| w_{kj}^{i} - \bar{w}_{kj}^{i}| = \left| \frac{ \sum_{\mathcal{A}\in\mathbb{A}_{(k,j)}} p(\mathcal{A}) }{ \sum_{\mathcal{A}\in\mathbb{A}} p(\mathcal{A}) } - \frac{ \sum_{\mathcal{A}\in\overline{\mathbb{A}}_{(k,j)}} p(\mathcal{A}) }{ \sum_{\mathcal{A}\in\overline{\mathbb{A}}} p(\mathcal{A}) } \right| \\
&= \left| (\sum_{\mathcal{A}\in\overline{\mathbb{A}}_{(k,j)}} p(\mathcal{A}))\left(  \frac{1}{ \sum_{\mathcal{A}\in\mathbb{A}} p(\mathcal{A}) } - \frac{1}{ \sum_{\mathcal{A}\in\overline{\mathbb{A}}} p(\mathcal{A}) } \right)  \right.\\
 &\hspace{4.6cm} + \left. \frac{ \sum_{\mathcal{A}\in\mathbb{A}_{(k,j)}\setminus \overline{\mathbb{A}}} p(\mathcal{A}) }{ \sum_{\mathcal{A}\in\mathbb{A}} p(\mathcal{A}) } \right| \\
&= \left| \frac{\sum_{\mathcal{A}\in\overline{\mathbb{A}}_{(k,j)}} p(\mathcal{A}) }{ \sum_{\mathcal{A}\in\overline{\mathbb{A}}} p(\mathcal{A}) } \left(  \frac{\sum_{\mathcal{A}\in\overline{\mathbb{A}}} p(\mathcal{A})}{ \sum_{\mathcal{A}\in\mathbb{A}} p(\mathcal{A}) } - 1  \right) \right. \\
&\hspace{4.6cm}\left. + \frac{ \sum_{\mathcal{A}\in\mathbb{A}_{(k,j)}\setminus \overline{\mathbb{A}}} p(\mathcal{A}) }{ \sum_{\mathcal{A}\in\mathbb{A}} p(\mathcal{A}) } \right| \\
&= \left| \bar{w}_{kj}^{i} \left(  - \frac{\sum_{\mathcal{A}\in\mathbb{A}\setminus \overline{\mathbb{A}}} p(\mathcal{A})}{ \sum_{\mathcal{A}\in\mathbb{A}} p(\mathcal{A}) }  \right) + \frac{ \sum_{\mathcal{A}\in\mathbb{A}_{(k,j)}\setminus \overline{\mathbb{A}}} p(\mathcal{A}) }{ \sum_{\mathcal{A}\in\mathbb{A}} p(\mathcal{A}) } \right| \\
&= \left| \frac{ \sum_{\mathcal{A}\in\mathbb{A}_{(k,j)}\setminus \overline{\mathbb{A}}} p(\mathcal{A}) }{ \sum_{\mathcal{A}\in\mathbb{A}} p(\mathcal{A}) } (1 - \bar{w}_{kj}^{i}) \right.\\
&\hspace{3.5cm}\left. - \frac{\sum_{\mathcal{A}\in(\mathbb{A}\setminus \overline{\mathbb{A}})\setminus\mathbb{A}_{(k,j)}} p(\mathcal{A})}{ \sum_{\mathcal{A}\in\mathbb{A}} p(\mathcal{A}) }\bar{w}_{kj}^{i}  \right| \\
&\leq  \frac{ \sum_{\mathcal{A}\in\mathbb{A}_{(k,j)}\setminus \overline{\mathbb{A}}} p(\mathcal{A}) }{ \sum_{\mathcal{A}\in\mathbb{A}} p(\mathcal{A}) } (1 - \bar{w}_{kj}^{i})\\
&\hspace{3.5cm} + \frac{\sum_{\mathcal{A}\in(\mathbb{A}\setminus \overline{\mathbb{A}})\setminus\mathbb{A}_{(k,j)}} p(\mathcal{A})}{ \sum_{\mathcal{A}\in\mathbb{A}} p(\mathcal{A}) }\bar{w}_{kj}^{i}  \\
\intertext{Given that the sets $\mathbb{A}_{(k,j)}\setminus \overline{\mathbb{A}}$ and $(\mathbb{A}\setminus \overline{\mathbb{A}})\setminus\mathbb{A}_{(k,j)}$ from the numerators are subsets of $\mathbb{A}\setminus \overline{\mathbb{A}}$}
&| w_{kj}^{i} - \bar{w}_{kj}^{i}|\; \leq\;  \gamma(\overline{\mathbb{A}}) (1 - \bar{w}_{kj}^{i}) + \gamma(\overline{\mathbb{A}})\bar{w}_{kj}^{i}  \\
&| w_{kj}^{i} - \bar{w}_{kj}^{i}|\; \leq\; \gamma(\overline{\mathbb{A}})
\end{align*}

\section*{ACKNOWLEDGMENT}
This work is supported by The Australian Government, via grant AUSMURIB000001 associated
with ONR MURI grant N00014-19-1-2571.


\bibliographystyle{IEEEtran.bst}  
\bibliography{references}

\end{document}